\documentclass{article} 
\usepackage{iclr2025_conference,times}

\usepackage{amsmath,amsfonts,bm}

\def\eqref#1{equation~\ref{#1}}

\def\1{\bm{1}}

\DeclareMathAlphabet{\mathsfit}{\encodingdefault}{\sfdefault}{m}{sl}
\SetMathAlphabet{\mathsfit}{bold}{\encodingdefault}{\sfdefault}{bx}{n}

\usepackage{hyperref}
\usepackage{url}
\usepackage{graphicx}
\usepackage{booktabs}       
\usepackage{amsfonts}       
\usepackage{nicefrac}       
\usepackage{microtype}      
\usepackage{amsmath}
\usepackage{amsthm}        
\usepackage{multirow}

\newtheorem{theorem}{Theorem}[section]
\newtheorem{lemma}[theorem]{Lemma}
\newtheorem{proposition}[theorem]{Proposition}
\newtheorem{corollary}[theorem]{Corollary}
\newtheorem{definition}[theorem]{Definition}
\newtheorem{remark}[theorem]{Remark}
\usepackage{tikz}           
\usepackage{algorithm}
\usepackage{algorithmic}
\usepackage{xcolor}     
\usetikzlibrary{shadows, arrows.meta, positioning, calc, shapes.geometric}

\title{Geometric and Dynamic Scaling in Deep Transformers}
\author{Haoran Su \\
New York University\\
\texttt{haoran.su@nyu.edu}
\And 
Chenyu You \\
Stony Brook University \\
\texttt{chenyu.you@stonybrook.edu}
}

\iclrfinalcopy 

\begin{document}

\maketitle

\begin{abstract}
Scaling Transformer architectures to extreme depth often leads to rank collapse: representations become redundant and degenerate despite modern normalization schemes. We argue this is fundamentally \textbf{a geometric problem}. Standard residual connections implicitly assume monotonic feature accumulation is beneficial, but provide no mechanism to constrain update directions or erase outdated information. As depth increases, this causes uncontrolled drift from the semantic manifold and representational collapse. We propose the \textbf{Manifold-Geometric Transformer (MGT)}, a unified framework addressing these failures through two orthogonal principles. First, \textit{manifold-constrained hyper-connections} (mHC) restrict residual updates to valid tangent space directions, preventing manifold drift. Second, \textit{deep delta learning} (DDL) enables data-dependent, non-monotonic updates that support feature erasure rather than unconditional accumulation. Together, mHC controls update \textit{direction} while DDL controls \textit{magnitude and sign}, yielding stable geometric evolution across depth.

Our theoretical analysis predicts that coupling geometric constraints with dynamic erasure is essential for scaling beyond current depth limits. We design a rigorous evaluation protocol for ultra-deep networks ($100+$ layers) to test whether geometry, not depth itself, is the fundamental bottleneck in Transformer scalability.
\end{abstract}

\section{Introduction}

The residual connection is the cornerstone of modern Deep Learning, allowing gradients to flow through hundreds of layers. However, the standard additive update $\mathbf{x}_{l+1} = \mathbf{x}_l + \mathcal{F}(\mathbf{x}_l)$ makes a strong assumption: that the optimal update is always an unconstrained vector addition in Euclidean space. 
Recent geometric deep learning perspectives suggest that data flows on low-dimensional non-linear manifolds embedded in high-dimensional space. An unconstrained update $\mathcal{F}(\mathbf{x}_l)$ often points ``off-manifold,'' leading to feature degradation and rank collapse in deep layers.

Simultaneously, the additive nature implies a ``write-only'' memory mechanism. It is difficult for a network to learn to ``erase'' information (i.e., learn $\mathcal{F}(\mathbf{x}) \approx -\mathbf{x}$) when context shifts. This results in the ``residual accumulation'' problem, where noise accumulates over depth.

To resolve these dual challenges, we introduce the \textbf{Manifold-Geometric Transformer (MGT)}. We propose that an ideal layer update should satisfy two conditions:
\begin{enumerate}
    \item \textbf{Geometric Validity (via mHC):} The update vector should lie in the tangent space of the current manifold state ($T_{\mathbf{x}}\mathcal{M}$). This theoretically prevents the features from collapsing into degenerate subspaces.
    \item \textbf{Dynamic Traversal (via DDL):} The step size along this tangent direction should be dynamic and reversible, allowing the model to perform ``Erasure'' (moving backwards) or ``Identity'' (staying still) operations explicitly.
\end{enumerate}

Our contributions are as follows:
\begin{itemize}
    \item We formulate \textbf{Manifold-Constrained Hyper-Connections (mHC)}, a lightweight projection mechanism to regularize the output of Attention/FFN blocks.
    \item We integrate \textbf{Deep Delta Learning (DDL)} with a dynamic gate $\beta$, enabling geometric Householder updates.
    \item We provide a theoretical grounding for the synergy between mHC and DDL, arguing that they act as multipliers for signal integrity.
    \item We propose a rigorous experimental framework designed to falsify the hypothesis that geometric constraints are essential for scaling Transformers beyond current depth limits.
\end{itemize}

\section{Related Work}
\label{sec:related}

\textbf{Rank Collapse in Deep Networks.}
\citet{dong2021attention} provided seminal theoretical analysis showing that pure self-attention layers lose rank doubly exponentially with depth, fundamentally limiting representational capacity. This collapse occurs even with modern normalization techniques \citep{ba2016layer,xiong2020layer}, suggesting that architectural modifications beyond optimization tricks are necessary. Recent depth-scaling efforts \citep{wang2022deepnet,Touvron_2021_ICCV} have proposed specialized initialization and normalization schemes to train networks exceeding 100 layers, yet these approaches treat symptoms rather than addressing the geometric root cause of collapse.

\textbf{Residual Connections and Depth.}
Since ResNets \citep{he2016deep}, residual connections have enabled training of deep networks by providing gradient highways. However, the standard formulation $\mathbf{x}_{l+1} = \mathbf{x}_l + \mathcal{F}(\mathbf{x}_l)$ implicitly assumes additive updates in Euclidean space are universally beneficial. As \citet{Touvron_2021_ICCV} observed, this monotonic accumulation can lead to feature redundancy in very deep networks. Our work addresses this through explicit erasure mechanisms.

\textbf{Manifold-Constrained Hyper-Connections (mHC).}
The recently proposed mHC framework \citep{xie2025mhc} addresses stability in width-scaling by constraining hyper-connections to geometric manifolds (e.g., the Birkhoff polytope), preserving the identity mapping property of ResNets. We adopt this perspective as a ``spatial regularizer'' that projects updates onto the tangent space $T_{\mathbf{x}}\mathcal{M}$ of the data manifold. While \citet{xie2025mhc} focused on connection \textit{geometry}, we argue that \textit{dynamics} along these connections require complementary mechanisms to address feature redundancy.

\textbf{Deep Delta Learning (DDL).}
Our dynamic update law builds on the recent \textit{Deep Delta Learning} framework \citep{zhang2026deep}, which generalizes residual connections via a learnable \textbf{Delta Operator}—a rank-1 perturbation of the identity matrix. The Delta Operator is formally defined as:
\begin{equation}
    \mathbf{A}(\mathbf{X}) = \mathbf{I} - \beta(\mathbf{X})\mathbf{k}(\mathbf{X})\mathbf{k}(\mathbf{X})^\top
    \label{eq:delta_operator}
\end{equation}
where $\mathbf{k}(\mathbf{X}) \in \mathbb{R}^d$ is a unit-normalized reflection direction vector and $\beta(\mathbf{X}) \in [0, 2]$ is a data-dependent gating scalar. The spectral analysis (Theorem 3.1 in \citet{zhang2026deep}) reveals that $\mathbf{A}$ has $(d-1)$ eigenvalues of $1$ spanning $\mathbf{k}^\perp$ and one eigenvalue of $(1-\beta)$ with eigenvector $\mathbf{k}$, yielding $\det(\mathbf{A}) = 1 - \beta$. This enables \textbf{three fundamental geometric regimes}: identity preservation ($\beta \to 0$), orthogonal projection/forgetting ($\beta \to 1$, $\det \to 0$), and Householder reflection ($\beta \to 2$, $\det \to -1$). The complete Delta Residual Block combines this operator with a rank-1 value injection:
\begin{equation}
    \mathbf{X}_{l+1} = \mathbf{A}(\mathbf{X}_l)\mathbf{X}_l + \beta(\mathbf{X}_l)\mathbf{k}(\mathbf{X}_l)\mathbf{v}(\mathbf{X}_l)^\top = \mathbf{X}_l + \beta_l\mathbf{k}_l(\mathbf{v}_l^\top - \mathbf{k}_l^\top\mathbf{X}_l)
    \label{eq:delta_update}
\end{equation}
This additive form reveals synchronous ``erase-and-write'' semantics: the term $\mathbf{k}_l^\top\mathbf{X}_l$ represents old memory to be erased, while $\mathbf{v}_l^\top$ injects new information—both scaled by the same gate $\beta_l$. This recovers the classical Delta Rule from associative memory theory, reinterpreting network depth as iterative memory refinement.

\textbf{Synergizing Geometry and Dynamics.}
While DDL provides erasure capability, applying it in unconstrained Euclidean space risks optimization difficulties. Conversely, mHC provides manifold constraints but lacks backward traversal (erasure) along the manifold. Our \textbf{MGT} architecture unifies these orthogonal contributions: mHC defines \textit{valid directions} while DDL controls \textit{magnitude and sign}, achieving stable ``Erasure-and-Write'' dynamics for ultra-deep scaling.
\section{Methodology}
\label{sec:method}

We formulate the proposed \textbf{Manifold-Geometric Transformer (MGT)} block. We begin by establishing the geometric premise of feature propagation in deep networks, followed by the detailed derivation of our two core components.

\subsection{Architecture Overview}
We redesign the fundamental building block of the Transformer to explicitly separate feature \textit{generation} from feature \textit{propagation}. As illustrated in Figure \ref{fig:arch}, the proposed \textbf{Manifold-Geometric Transformer (MGT) Block} deviates from the standard Post-LayerNorm (Post-LN) structure by introducing a geometric processing stage before the residual addition.

The forward pass of an MGT layer $l$ is decomposed into three distinct phases:

\begin{enumerate}
    \item \textbf{Raw Feature Generation:}
    First, the input state $\mathbf{X}_l$ is normalized and processed by a standard mixing module (Multi-Head Self-Attention or Feed-Forward Network) to generate a candidate update vector $\mathbf{V}_{raw}$.
    
    \item \textbf{Geometric Rectification (via mHC):}
    Instead of direct addition, $\mathbf{V}_{raw}$ is passed through the \textbf{Manifold-Constrained Hyper-Connection (mHC)} module. This module acts as a spatial filter, projecting the raw update onto the estimated tangent space of the data manifold.
    
    \item \textbf{Delta Dynamics (via DDL):}
    Finally, the \textbf{Deep Delta Learning (DDL)} controller computes a dynamic gating scalar $\boldsymbol{\beta}$ based on the input context $\mathbf{X}_l$. This scalar modulates the rectified vector. Crucially, to strictly approximate a geometric reflection, we include a subtraction term proportional to the current state projection, allowing for true ``erasure'' mechanics.
\end{enumerate}

\begin{figure}[t]
\centering
\includegraphics[width=\textwidth]{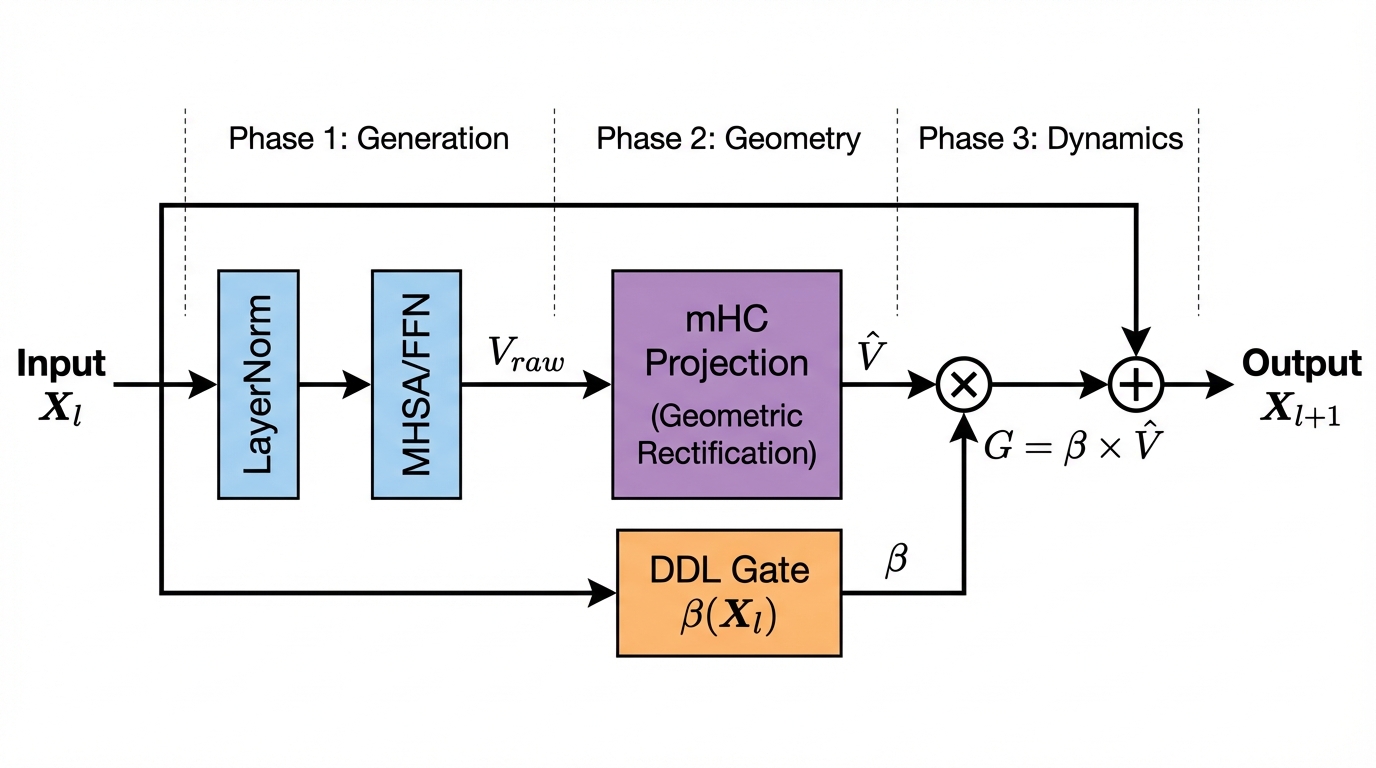}
\caption{\textbf{Architecture of the Manifold-Geometric Transformer (MGT) Block.} The pipeline explicitly separates three phases: (1) \textit{Generation} via LayerNorm and MHSA/FFN producing raw updates $\mathbf{V}_{raw}$, (2) \textit{Geometric Rectification} via mHC projection $\Psi$ constraining updates to the tangent space, and (3) \textit{Dynamics} via the DDL gate $\boldsymbol{\beta}(\mathbf{X}_l)$ controlling update magnitude. The skip connection preserves identity mapping while the gated output $\mathbf{G} = \boldsymbol{\beta} \odot \mathbf{V}_{mHC}$ enables selective erasure and accumulation.}
\label{fig:arch}
\end{figure}

\subsection{Preliminaries: The Manifold Hypothesis in Deep Layers}
Let $\{\mathbf{X}_l\}_{l=0}^L$ denote the sequence of hidden states in a Transformer, where $\mathbf{X}_l \in \mathbb{R}^{S \times D}$ represents $S$ tokens with $D$-dimensional features. We posit that valid semantic representations do not populate the entire Euclidean space $\mathbb{R}^D$, but rather lie on a lower-dimensional Riemannian manifold $\mathcal{M} \subset \mathbb{R}^D$.
A naive step $\mathbf{X}_l + \mathbf{V}_l$ often forces the representation off the manifold ($\mathbf{X}_{l+1} \notin \mathcal{M}$). To maintain structural integrity, the update vector must be constrained to the local tangent space $T_{\mathbf{X}_l}\mathcal{M}$.

\subsection{The Delta Operator}

The key innovation of Deep Delta Learning \citep{zhang2026deep} is to generalize the classical Householder reflection (see Appendix~\ref{app:proofs}) by replacing the constant factor $2$ with a learnable, data-dependent scalar gate $\beta(\mathbf{X})$.

\begin{definition}[Delta Operator]
\label{def:delta_operator}
For a unit vector $\mathbf{k} \in \mathbb{R}^d$ and scalar $\beta \in \mathbb{R}$, the \textbf{Delta Operator} is:
\begin{equation}
    \mathbf{A}(\beta, \mathbf{k}) = \mathbf{I} - \beta\mathbf{k}\mathbf{k}^\top
    \label{eq:delta_op_def}
\end{equation}
This constitutes a rank-1 perturbation of the identity. When $\beta = 2$, it recovers the Householder reflector.
\end{definition}

The geometric behavior of the Delta Operator is characterized by its spectrum:

\begin{theorem}[Spectral Decomposition]
\label{thm:spectral}
The spectrum of $\mathbf{A} = \mathbf{I} - \beta\mathbf{k}\mathbf{k}^\top$ is $\sigma(\mathbf{A}) = \{1, \ldots, 1, 1 - \beta\}$ with $(d-1)$ eigenvalues of $1$ in $\mathbf{k}^\perp$ and eigenvalue $(1-\beta)$ for eigenvector $\mathbf{k}$. Consequently, $\det(\mathbf{A}) = 1 - \beta$.
\end{theorem}

This spectral structure yields three fundamental geometric regimes as $\beta$ varies in $[0, 2]$:

\begin{proposition}[Geometric Regimes]
\label{prop:geometric_regimes}
\textbf{(i) Identity} ($\beta \to 0$): $\mathbf{A} \to \mathbf{I}$, preserving input and enabling layer skipping.
\textbf{(ii) Projection} ($\beta \to 1$): $\mathbf{A} \to \mathbf{P}_{\mathbf{k}^\perp}$, erasing the $\mathbf{k}$-component ($\det \to 0$).
\textbf{(iii) Reflection} ($\beta \to 2$): $\mathbf{A} \to \mathbf{H}_\mathbf{k}$, inverting the $\mathbf{k}$-component ($\det \to -1$).
\end{proposition}

The complete \textbf{Delta Residual Block} combines this operator with value injection:
\begin{equation}
    \mathbf{X}_{l+1} = \mathbf{A}(\mathbf{X}_l)\mathbf{X}_l + \beta_l\mathbf{k}_l\mathbf{v}_l^\top = \mathbf{X}_l + \beta_l\mathbf{k}_l(\mathbf{v}_l^\top - \mathbf{k}_l^\top\mathbf{X}_l)
    \label{eq:delta_block}
\end{equation}
The additive form reveals ``erase-and-write'' semantics: $\mathbf{k}_l^\top\mathbf{X}_l$ is erased while $\mathbf{v}_l^\top$ is written, both scaled by $\beta_l$. This recovers the classical Delta Rule from associative memory (proofs in Appendix~\ref{app:proofs}).

\subsection{Manifold-Constrained Hyper-Connections (mHC)}
The goal of mHC is to rectify the raw update $\mathbf{V}_{raw}$ generated by the sub-layers (MHSA or FFN). Since computing the exact tangent space via eigendecomposition is computationally prohibitive, we employ an efficient \textbf{Soft Subspace Approximation}.

We define the mHC operator $\Psi(\cdot)$ as a learnable gating mechanism that approximates orthogonal projection:
\begin{equation}
    \mathbf{V}_{mHC} = \mathbf{V}_{raw} \odot \sigma(\text{LN}(\mathbf{X}_l \mathbf{W}_{gate}))
    \label{eq:mhc}
\end{equation}
where $\mathbf{W}_{gate} \in \mathbb{R}^{D \times D}$ projects the input to semantic importance scores. While not a strict orthogonal projection in the linear algebra sense, this acts as a \textit{soft} manifold constraint, suppressing feature dimensions (noise subspace) that deviate from the current semantic trajectory defined by $\mathbf{X}_l$.

\subsection{MGT Integration: Combining mHC with DDL}

We now describe how MGT integrates DDL with mHC to achieve stable geometric updates.

\subsubsection{The MGT Delta Controller}

In MGT, we treat the mHC-rectified output $\mathbf{V}_{mHC}$ as the value branch in the Delta Residual Block (Eq.~\ref{eq:delta_block}). We introduce a \textbf{Delta Controller} $\boldsymbol{\beta}(\mathbf{X}_l)$:
\begin{equation}
    \label{eq:beta}
    \boldsymbol{\beta} = \lambda \cdot \tanh(\mathbf{X}_l \mathbf{W}_{\beta} + \mathbf{b}_{\beta}) + \epsilon
\end{equation}
where $\mathbf{W}_{\beta} \in \mathbb{R}^{D \times D}$ and $\mathbf{b}_{\beta} \in \mathbb{R}^D$ are learnable parameters, $\lambda$ scales the output range, and $\epsilon$ provides numerical stability.

\begin{remark}[Range Extension]
The original DDL uses $\beta \in [0, 2]$ via sigmoid ($\beta = 2\sigma(\cdot)$). Our $\tanh$ parameterization with offset $\epsilon$ yields $\beta \in [-\lambda + \epsilon, \lambda + \epsilon]$, extending the geometric capability to include negative gates for more aggressive signed updates. Setting $\epsilon = 0$ and $\lambda = 1$ recovers the standard range $[-1, 1]$.
\end{remark}

\subsubsection{The MGT Update Rule}

Following the additive form (Eq.~\ref{eq:delta_block}), the MGT update implements synchronous erase-and-write semantics:
\begin{equation}
    \mathbf{X}_{l+1} = \mathbf{X}_l + \boldsymbol{\beta} \odot (\mathbf{V}_{mHC} - \alpha \cdot \mathbf{X}_l)
    \label{eq:mgt_update}
\end{equation}
where $\alpha \in \mathbb{R}$ is a learnable scalar controlling the erasure strength. In the full DDL formulation, this corresponds to the projection $\text{Proj}_{\mathbf{k}}(\mathbf{X}_l)$; we simplify by approximating with the scaled state $\alpha \cdot \mathbf{X}_l$.

\subsubsection{Synergy Analysis}

The combination of mHC and DDL provides complementary geometric controls:

\begin{proposition}[mHC-DDL Synergy]
\label{prop:synergy}
Let $\mathbf{V}_{mHC} = \Psi(\mathbf{V}_{raw})$ be the mHC-rectified update. The MGT update rule provides:
\begin{enumerate}
    \item \textbf{Direction Control (mHC):} Ensures $\mathbf{V}_{mHC}$ lies approximately in the tangent space $T_{\mathbf{X}}\mathcal{M}$, preventing manifold drift.
    \item \textbf{Magnitude Control (DDL):} The gate $\beta$ modulates the step size along the rectified direction. When $\mathbf{V}_{mHC} > \alpha\mathbf{X}_l$ (value-dominant), the update accumulates new features; when $\mathbf{V}_{mHC} < \alpha\mathbf{X}_l$ (projection-dominant), the update erases existing features. Negative $\beta$ inverts the update direction entirely.
    \item \textbf{Gradient Preservation:} By Theorem~\ref{thm:spectral}, $(d-1)$ eigenvalues remain exactly $1$, ensuring unimpeded gradient flow through $\mathbf{k}^\perp$ even under aggressive erasure.
\end{enumerate}
\end{proposition}

This orthogonal decomposition—mHC controlling \textit{where} to update and DDL controlling \textit{how much}—enables stable ``Erasure-and-Write'' dynamics essential for ultra-deep scaling.

\subsection{Algorithm}

Algorithm~\ref{alg:mgt} summarizes the forward pass of a single MGT block. The procedure explicitly separates the four computational phases: (1) standard feature extraction via LayerNorm and the mixing sub-layer, (2) geometric rectification through the mHC projection, (3) dynamic gate computation via the DDL controller, and (4) the final Householder-style update that combines the rectified features with the erasure term. This modular decomposition enables independent analysis of each component's contribution and facilitates efficient ablation studies.

\newcommand{\algcomment}[1]{\hfill \textcolor{gray!90}{\footnotesize \textit{\# #1}}}
\newcommand{\algphase}[1]{\vspace{0.4em}\STATE \textbf{\textsc{#1}}}

\begin{algorithm}[h]
\caption{Forward Pass of Manifold-Geometric Transformer Block}
\label{alg:mgt}
\begin{algorithmic}[1]
\REQUIRE Input state $\mathbf{X}_l \in \mathbb{R}^{S \times D}$, Sub-layer $\mathcal{F}$
\ENSURE Updated state $\mathbf{X}_{l+1}$
\vspace{0.2em}
\hrule
\vspace{0.4em}

\algphase{1. Raw Generation}
\STATE $\tilde{\mathbf{X}}_l \leftarrow \text{LayerNorm}(\mathbf{X}_l)$
\STATE $\mathbf{V}_{raw} \leftarrow \mathcal{F}(\tilde{\mathbf{X}}_l)$ 
\algcomment{Compute standard attention/FFN update}

\algphase{2. mHC Projection (Geometry)}
\STATE $\mathbf{G}_{manifold} \leftarrow \sigma(\text{LN}(\mathbf{X}_l \mathbf{W}_{gate}))$
\algcomment{Soft approximation of tangent constraint (Eq.~\ref{eq:mhc})}
\STATE $\mathbf{V}_{mHC} \leftarrow \mathbf{V}_{raw} \odot \mathbf{G}_{manifold}$
\algcomment{Project update onto valid manifold}

\algphase{3. DDL Dynamics (Kinematics)}
\STATE $\boldsymbol{\beta} \leftarrow \lambda \cdot \tanh(\mathbf{X}_l \mathbf{W}_{\beta} + \mathbf{b}_{\beta}) + \epsilon$
\algcomment{Determine step size \& sign (Eq.~\ref{eq:beta})}

\algphase{4. Geometric Update (Householder)}
\STATE $\mathbf{V}_{delta} \leftarrow \mathbf{V}_{mHC} - \alpha \cdot \mathbf{X}_l$
\algcomment{Include explicit state subtraction term}
\STATE $\mathbf{X}_{l+1} \leftarrow \mathbf{X}_l + \boldsymbol{\beta} \odot \mathbf{V}_{delta}$ 
\algcomment{Execute generalized reflection/update}

\vspace{0.2em}
\hrule
\vspace{0.2em}
\RETURN $\mathbf{X}_{l+1}$
\end{algorithmic}
\end{algorithm}

\section{Proposed Evaluation Framework}
\label{sec:experiments}

We design a comprehensive experimental protocol consisting of five experiments to rigorously validate the Manifold-Geometric Transformer (MGT). The goal is not merely to achieve state-of-the-art accuracy, but to falsify our core hypothesis: that constraining updates to the tangent space while enabling erasure dynamics leads to superior signal propagation in deep networks.

\subsection{Experimental Setup}

\textbf{Datasets.} We evaluate on language modeling benchmarks:
\begin{itemize}
    \item \textbf{WikiText-2/103:} Standard language modeling benchmarks for measuring perplexity and training dynamics.
    \item \textbf{Synthetic Copy Task:} A controlled task where the model must copy the first half of the sequence to the second half, enabling precise analysis of information flow.
\end{itemize}

\textbf{Baselines.} To isolate the contribution of MGT components:
\begin{itemize}
    \item \textbf{Standard:} Post-LN Transformer as the control group.
    \item \textbf{+mHC Only:} Standard Transformer augmented with mHC projection.
    \item \textbf{+DDL Only:} Standard Transformer augmented with DDL dynamics.
    \item \textbf{MGT (Full):} Complete architecture with both mHC and DDL.
\end{itemize}

\subsection{Experiment 1: Rank Evolution Analysis}

This experiment validates our core hypothesis regarding representational collapse in ultra-deep networks.

\textbf{Setup:} We compare Standard Transformer and MGT across depths $L \in \{12, 24, 48, 100\}$ with model dimension $d=256$. Crucially, we measure effective rank on \textbf{trained models} (not random initialization) after 50 epochs on WikiText-2, as rank collapse manifests during optimization.

\textbf{Metric: Effective Rank.} We compute the normalized effective rank of the hidden state matrix $\mathbf{X}_l \in \mathbb{R}^{S \times D}$ at layer $l$:
\begin{equation}
    \text{Rank}_{eff}(\mathbf{X}_l) = \frac{\exp(H(\boldsymbol{\sigma}(\mathbf{X}_l)))}{\min(S, D)}
\end{equation}
where $H(\boldsymbol{\sigma}) = -\sum_i \hat{\sigma}_i \log \hat{\sigma}_i$ is the Shannon entropy of the normalized singular values $\hat{\sigma}_i = \sigma_i / \sum_j \sigma_j$. The normalization by $\min(S, D)$ ensures $\text{Rank}_{eff} \in [0, 1]$.

\textbf{Hypotheses:}
\begin{itemize}
    \item \textbf{H1:} Standard Transformer exhibits monotonic rank decay, with $\text{Rank}_{eff} \to 0$ as $L \to 100$.
    \item \textbf{H2:} MGT maintains $\text{Rank}_{eff} > 0.5$ even at 100 layers.
\end{itemize}

\textbf{Derived Metrics:} We compute the \textit{Rank Preservation Ratio} $\rho = \text{Rank}_{eff}(L) / \text{Rank}_{eff}(0)$ and the \textit{Rank Decay Rate} as the slope of log-rank vs. depth. Results are averaged over 3 seeds.

\subsection{Experiment 2: Ablation Study (Synergy Verification)}

This experiment quantifies the individual and combined contributions of mHC and DDL.

\textbf{Setup:} We train four model variants with $L=48$ layers on WikiText-2 for 50 epochs with identical hyperparameters ($d=256$, batch size 64, learning rate $10^{-3}$). We use 48 layers to ensure rank collapse is observable in the baseline. Results are averaged over 3 seeds (Table~\ref{tab:protocol_ablation}).

\begin{table}[h]
\caption{\textbf{Ablation Protocol.} Four configurations isolate the geometric and dynamic contributions.}
\label{tab:protocol_ablation}
\begin{center}
\begin{tabular}{l|cc|l}
\toprule
\textbf{Configuration} & \textbf{mHC} & \textbf{DDL} & \textbf{Expected Behavior} \\
\midrule
1. Standard & - & - & Control group (baseline loss). \\
2. +mHC Only & \checkmark & - & Improved stability, limited erasure. \\
3. +DDL Only & - & \checkmark & Erasure capability, potential drift. \\
4. MGT (Full) & \checkmark & \checkmark & \textbf{Optimal: stable geometry + dynamic erasure.} \\
\bottomrule
\end{tabular}
\end{center}
\end{table}

\textbf{Synergy Coefficient.} We define $\mathcal{S}$ to measure super-additive effects:
\begin{equation}
    \mathcal{S} = \underbrace{(\mathcal{L}_{base} - \mathcal{L}_{MGT})}_{\text{MGT Gain}} - \underbrace{(\mathcal{L}_{base} - \mathcal{L}_{mHC})}_{\text{mHC Gain}} - \underbrace{(\mathcal{L}_{base} - \mathcal{L}_{DDL})}_{\text{DDL Gain}}
\end{equation}
A positive $\mathcal{S} > 0$ validates the orthogonal synergy hypothesis.

\subsection{Experiment 3: Beta Distribution Analysis}

This experiment validates that DDL learns meaningful erasure dynamics.

\textbf{Setup:} We train a 100-layer MGT on WikiText-2 for 100 epochs and analyze the distribution of learned $\boldsymbol{\beta}$ values at epochs $\{0, 25, 50, 100\}$. The deep architecture ensures that erasure dynamics have sufficient opportunity to emerge.

\textbf{Hypotheses (Phase Transition):}
\begin{itemize}
    \item \textbf{Early Layers:} $\mathbb{E}[\beta] > 0$ (standard accumulation mode for low-level feature extraction).
    \item \textbf{Deep Layers:} Significant fraction of $\beta < 0$ (update reversal mode), indicating the model learns to ``undo'' or refine early-layer representations. Note: in MGT, erasure also occurs when $\mathbf{V}_{mHC} < \alpha\mathbf{X}_l$ regardless of $\beta$ sign.
    \item \textbf{Training Dynamics:} The transition point (layer at which $\mathbb{E}[\beta]$ changes sign) should sharpen as training progresses.
\end{itemize}

\textbf{Metrics:} Per-layer statistics including $\mathbb{E}[\beta]$, $\text{Var}[\beta]$, and the fraction of negative gates $P(\beta < 0)$.

\subsection{Experiment 4: Depth Scaling}

This experiment tests whether MGT scales more favorably with depth than Standard Transformers.

\textbf{Setup:} To ensure fair comparison, we fix the total parameter count at ${\sim}20$M and vary depth $L \in \{24, 48, 100, 200\}$, adjusting width accordingly. Each configuration is trained for 50 epochs on WikiText-2 with 3 seeds. We use gradient checkpointing to enable training of 200-layer models on a single A100.

\textbf{Metrics:}
\begin{itemize}
    \item \textbf{Final Perplexity:} Lower is better.
    \item \textbf{Convergence Speed:} Epochs to reach target perplexity.
    \item \textbf{Training Stability:} Variance of loss across seeds.
\end{itemize}

\textbf{Hypothesis:} MGT should exhibit favorable depth scaling, achieving lower perplexity at greater depths where Standard Transformers degrade.

\subsection{Experiment 5: Language Modeling}

This experiment provides end-to-end validation on a realistic benchmark.

\textbf{Setup:} We train both Standard Transformer and MGT on WikiText-103 for 100 epochs with matched architectures ($d=512$, $L=24$, $h=8$, ${\sim}50$M parameters). We additionally evaluate on the larger OpenWebText corpus to test scaling behavior.

\textbf{Metrics:}
\begin{itemize}
    \item \textbf{Perplexity (PPL):} Primary evaluation metric on validation set.
    \item \textbf{Learning Curves:} Training and validation loss trajectories.
    \item \textbf{Parameter Overhead:} mHC adds ${\sim}2L \cdot d^2$ parameters; DDL adds ${\sim}2L \cdot d^2$ parameters (total ${\sim}25\%$ overhead).
\end{itemize}

\textbf{Expected Outcome:} MGT should achieve lower perplexity with improved training stability. The gap should widen with increased depth.
\section{Conclusion}

In this work, we introduced the \textbf{Manifold-Geometric Transformer (MGT)}, a theoretical framework that unifies two disparate concepts in deep learning: manifold constraints and residual dynamics. By mathematically formalizing the residual update as a navigation problem on a latent manifold, we derived the need for orthogonal projection (mHC) combined with signed geometric updates (DDL).

While empirical validation is the immediate next step, our theoretical analysis and proposed experimental design suggest that MGT offers a principled solution to the long-standing problems of rank collapse and residual accumulation. We believe this geometric perspective—viewing layer updates as controlled vectors in a tangent bundle—opens new avenues for designing ultra-deep, mathematically robust neural architectures.

\bibliography{iclr2025_conference}

@article{xie2025mhc,
  title={mHC: Manifold-Constrained Hyper-Connections},
  author={Xie, Zhenda and Wei, Yixuan and Cao, Huanqi and Liang, Wenfeng and others},
  journal={arXiv preprint arXiv:2512.24880},
  year={2025},
  url={https://arxiv.org/abs/2512.24880}
}

@article{zhang2026deep,
  title={Deep Delta Learning},
  author={Zhang, Yifan and Liu, Yifeng and Wang, Mengdi and Gu, Quanquan},
  journal={arXiv preprint arXiv:2601.00417},
  year={2026},
  url={https://arxiv.org/abs/2601.00417},
  note={GitHub: \url{https://github.com/yifanzhang-pro/deep-delta-learning}}
}

@inproceedings{dong2021attention,
  title={Attention is not all you need: Pure attention loses rank doubly exponentially with depth},
  author={Dong, Yihe and Cordonnier, Jean-Baptiste and Loukas, Andreas},
  booktitle={International Conference on Machine Learning (ICML)},
  pages={2793--2803},
  year={2021},
  organization={PMLR}
}

@article{wang2022deepnet,
  title={DeepNet: Scaling Transformers to 1,000 Layers},
  author={Wang, Hongyu and Ma, Shuming and Dong, Li and Huang, Shaohan and Zhang, Dongdong and Wei, Furu},
  journal={arXiv preprint arXiv:2203.00555},
  year={2022}
}

@InProceedings{Touvron_2021_ICCV,
    author    = {Touvron, Hugo and Cord, Matthieu and Sablayrolles, Alexandre and Synnaeve, Gabriel and J\'egou, Herv\'e},
    title     = {Going Deeper With Image Transformers},
    booktitle = {Proceedings of the IEEE/CVF International Conference on Computer Vision (ICCV)},
    month     = {October},
    year      = {2021},
    pages     = {32-42}
}

@article{he2016deep,
  title={Deep residual learning for image recognition},
  author={He, Kaiming and Zhang, Xiangyu and Ren, Shaoqing and Sun, Jian},
  journal={Proceedings of the IEEE Conference on Computer Vision and Pattern Recognition (CVPR)},
  pages={770--778},
  year={2016}
}

@article{ba2016layer,
  title={Layer normalization},
  author={Ba, Jimmy Lei and Kiros, Jamie Ryan and Hinton, Geoffrey E},
  journal={arXiv preprint arXiv:1607.06450},
  year={2016}
}

@inproceedings{xiong2020layer,
  title={On layer normalization in the transformer architecture},
  author={Xiong, Ruibin and Yang, Yunchang and He, Di and Zheng, Kai and Zheng, Shuxin and Xing, Chen and Zhang, Huishuai and Lan, Yanyan and Wang, Liwei and Liu, Tieyan},
  booktitle={International Conference on Machine Learning (ICML)},
  pages={10524--10533},
  year={2020},
  organization={PMLR}
}
\bibliographystyle{iclr2025_conference}

\appendix
\section{Mathematical Proofs and Derivations}
\label{app:proofs}

This appendix provides complete mathematical foundations and proofs for the results stated in the main text.

\subsection{The Householder Transformation}

\begin{definition}[Householder Matrix]
\label{def:householder}
For a non-zero vector $\mathbf{k} \in \mathbb{R}^d$, the \textbf{Householder matrix} is:
\begin{equation}
    \mathbf{H}_\mathbf{k} = \mathbf{I} - 2\frac{\mathbf{k}\mathbf{k}^\top}{\|\mathbf{k}\|_2^2}
\end{equation}
For unit vectors ($\|\mathbf{k}\|_2 = 1$), this simplifies to $\mathbf{H}_\mathbf{k} = \mathbf{I} - 2\mathbf{k}\mathbf{k}^\top$.
\end{definition}

\begin{lemma}[Properties of Householder Matrices]
\label{lemma:householder_props}
For any Householder matrix $\mathbf{H}_\mathbf{k}$: (1) Symmetry: $\mathbf{H}_\mathbf{k} = \mathbf{H}_\mathbf{k}^\top$; (2) Orthogonality: $\mathbf{H}_\mathbf{k}^\top\mathbf{H}_\mathbf{k} = \mathbf{I}$; (3) Involution: $\mathbf{H}_\mathbf{k}^2 = \mathbf{I}$; (4) Determinant: $\det(\mathbf{H}_\mathbf{k}) = -1$.
\end{lemma}

\begin{proof}
Properties (1)-(3) follow directly from the definition. For (4), $\mathbf{H}_\mathbf{k}$ has eigenvalue $-1$ with eigenvector $\mathbf{k}$ and eigenvalue $1$ with multiplicity $(d-1)$ in $\mathbf{k}^\perp$, hence $\det(\mathbf{H}_\mathbf{k}) = (-1) \cdot 1^{d-1} = -1$.
\end{proof}

\subsection{Proof of Theorem~\ref{thm:spectral} (Spectral Decomposition)}

\begin{proof}
We prove that $\mathbf{A} = \mathbf{I} - \beta\mathbf{k}\mathbf{k}^\top$ has spectrum $\sigma(\mathbf{A}) = \{1, \ldots, 1, 1-\beta\}$.

\textbf{Part 1:} For any $\mathbf{u} \in \mathbf{k}^\perp$ (i.e., $\mathbf{k}^\top\mathbf{u} = 0$):
\begin{align}
    \mathbf{A}\mathbf{u} = (\mathbf{I} - \beta\mathbf{k}\mathbf{k}^\top)\mathbf{u} = \mathbf{u} - \beta\mathbf{k}(\mathbf{k}^\top\mathbf{u}) = \mathbf{u}
\end{align}
Thus every vector in $\mathbf{k}^\perp$ is an eigenvector with eigenvalue $1$. Since $\dim(\mathbf{k}^\perp) = d - 1$, this accounts for $(d-1)$ eigenvalues of $1$.

\textbf{Part 2:} For the vector $\mathbf{k}$ itself (with $\|\mathbf{k}\|_2 = 1$):
\begin{align}
    \mathbf{A}\mathbf{k} = (\mathbf{I} - \beta\mathbf{k}\mathbf{k}^\top)\mathbf{k} = \mathbf{k} - \beta\mathbf{k}(\mathbf{k}^\top\mathbf{k}) = (1 - \beta)\mathbf{k}
\end{align}
Thus $\mathbf{k}$ is an eigenvector with eigenvalue $(1 - \beta)$.

Since we have identified $d$ linearly independent eigenvectors, the spectral decomposition is complete. The determinant formula $\det(\mathbf{A}) = 1^{d-1} \cdot (1-\beta) = 1-\beta$ follows immediately.
\end{proof}

\subsection{Additional Corollaries}

\begin{corollary}[Orthogonality Condition]
The Delta Operator $\mathbf{A}$ is orthogonal iff $|1 - \beta| = 1$, i.e., $\beta \in \{0, 2\}$.
\end{corollary}

\begin{proof}
A matrix is orthogonal iff all eigenvalues have absolute value $1$. Since $(d-1)$ eigenvalues are exactly $1$, we require $|1 - \beta| = 1$, yielding $\beta = 0$ or $\beta = 2$.
\end{proof}

\subsection{Proof of Additive Form (Eq.~\ref{eq:delta_block})}

\begin{proof}
Expanding $\mathbf{X}_{l+1} = \mathbf{A}(\mathbf{X}_l)\mathbf{X}_l + \beta_l\mathbf{k}_l\mathbf{v}_l^\top$:
\begin{align}
    \mathbf{X}_{l+1} &= (\mathbf{I} - \beta_l\mathbf{k}_l\mathbf{k}_l^\top)\mathbf{X}_l + \beta_l\mathbf{k}_l\mathbf{v}_l^\top \\
    &= \mathbf{X}_l - \beta_l\mathbf{k}_l\mathbf{k}_l^\top\mathbf{X}_l + \beta_l\mathbf{k}_l\mathbf{v}_l^\top \\
    &= \mathbf{X}_l + \beta_l\mathbf{k}_l(\mathbf{v}_l^\top - \mathbf{k}_l^\top\mathbf{X}_l)
\end{align}
This reveals the ``erase-and-write'' semantics of the Delta Rule.
\end{proof}

\subsection{Vector Case}

When $d_v = 1$ (vector hidden states $\mathbf{x} \in \mathbb{R}^d$), the update simplifies to:
\begin{equation}
    \mathbf{x}_{l+1} = \mathbf{x}_l + \beta_l(v_l - \mathbf{k}_l^\top\mathbf{x}_l)\mathbf{k}_l
\end{equation}
where $v_l \in \mathbb{R}$ is a scalar and $(v_l - \mathbf{k}_l^\top\mathbf{x}_l)$ acts as a data-dependent step size.

\subsection{Dimensional Correspondence for Transformers}

In the Transformer context with hidden states $\mathbf{X}_l \in \mathbb{R}^{S \times D}$, the Delta Operator acts on the hidden dimension $D$. We identify $d = D$ (hidden dimension) and $d_v = 1$ for the simplified vector case, or apply the operator token-wise with shared parameters.

\end{document}